\documentclass[11pt]{article}

\usepackage[a4paper,margin=1in]{geometry}
\usepackage[T1]{fontenc}
\usepackage{lmodern}

\usepackage[numbers,sort&compress]{natbib}

\usepackage{microtype}
\usepackage{booktabs}
\usepackage{amsmath,amssymb,amsfonts,amsthm}
\usepackage{bm}
\usepackage{enumitem}
\usepackage{needspace}

\IfFileExists{orcidlink.sty}{
  \usepackage{orcidlink}
}{
  \newcommand{\orcidlink}[1]{}
}

\DeclareMathOperator*{\argmin}{arg\,min}
\DeclareMathOperator{\dist}{dist}

\newcommand{\Eta}{\mathcal{E}}

\newcommand{\R}{\mathbb{R}}
\newcommand{\Hcal}{\mathcal{H}}
\newcommand{\Lcal}{\mathcal{L}}
\newcommand{\Rcal}{\mathcal{R}}
\newcommand{\Dcal}{\mathcal{D}}
\newcommand{\Ccal}{\mathcal{C}}

\theoremstyle{plain}
\newtheorem{theorem}{Theorem}
\newtheorem{proposition}{Proposition}
\newtheorem{lemma}{Lemma}
\newtheorem{corollary}{Corollary}

\theoremstyle{definition}
\newtheorem{definition}{Definition}
\newtheorem{assumption}{Assumption}

\theoremstyle{remark}
\newtheorem{remark}{Remark}

\theoremstyle{definition}

\title{A Mosco sufficient condition for intrinsic stability of non-unique convex Empirical Risk Minimization}

\usepackage{authblk}
\usepackage{hyperref}
\author[1]{Karim Bounja\thanks{Corresponding author. Email: \texttt{k.bounja.doc@uhp.ac.ma}}}
\author[2]{Lahcen Laayouni}
\author[1]{Abdeljalil Sakat}

\affil[1]{Laboratory for Analysis and Modeling of Systems and Decision Support (LAMSAD), Hassan 1st University of Settat, Settat 26000, Morocco}
\affil[2]{Department of Computer Science, School of Science and Engineering, Al Akhawayn University in Ifrane, Ifrane 53000, Morocco}

\date{}

\begin{document}
\maketitle
\vspace{-2.8em}  

\begin{abstract}
Empirical risk minimization (ERM) stability is usually studied via single-valued outputs, while convex non-strict losses yield set-valued minimizers. We identify Painlev\'e--Kuratowski upper semicontinuity (PK-u.s.c.) as the intrinsic stability notion for the ERM solution correspondence (set-level Hadamard well-posedness) and a prerequisite to interpret stability of selections. We then characterize a minimal non-degenerate qualitative regime: Mosco-consistent perturbations and locally bounded minimizers imply PK-u.s.c., minimal-value continuity, and consistency of vanishing-gap near-minimizers. Quadratic growth yields explicit quantitative deviation bounds.
\end{abstract}


\section{Introduction}

Empirical risk minimization (ERM) is a cornerstone of modern learning, underlying a wide range of estimation procedures in statistics and machine learning \citep{Vapnik1998,ShalevShwartz2014}. It consists in minimizing, for each dataset $\mathcal D$, a data-dependent loss functional $\mathcal L_{\mathcal D}:\mathcal H\to\mathbb R\cup \{+\infty\}$ over $f\in\mathcal H$, where $\mathcal H$ is a real Hilbert space, and we use the decomposition $\mathcal L_{\mathcal D}(f)=\mathcal F(f)+\Psi(\mathcal D,f)$. In many contemporary regimes, objectives are convex but not strictly convex, so the solution map may be set-valued \citep{rockafellar1998}. Empirically, lack of reproducibility is often reported in modern pipelines: small perturbations of the data, initialization, data ordering, or optimization details can lead to markedly different predictors with comparable training performance \citep{Pineau2021,Zhang2017}. Related phenomena arise in overparameterized settings, where many near-interpolating solutions may coexist and differ substantially in norm or other geometric characteristics \citep{Belkin2019,Neyshabur2015}. Our contribution is to propose a problem-intrinsic notion of stability for (possibly non-unique) convex ERM—formulated at the level of the solution set—so as to distinguish variability stemming from the ERM problem itself from variability induced by a particular selection or solver.

In variational terms, it is well known that minimizer mappings of convex functionals
need not be stable under arbitrarily small perturbations unless additional structural
assumptions---most commonly compactness or (equi-)coercivity---are imposed
\citep{attouch1984,dalmaso1993,rockafellar1998}.
In particular, neither pointwise nor uniform convergence of convex objectives
generally ensures convergence of minimizers \citep{rockafellar1998}.
This exposes a conceptual gap in learning theory once ERM is non-unique: the natural object is the solution correspondence
$S:\mathbb D \rightrightarrows \Hcal$, $S(\Dcal):=\Ccal_{\Dcal}:=\arg\min \Lcal_{\Dcal}$,
and a natural well-posedness requirement is Painlev\'e--Kuratowski upper semicontinuity (PK-u.s.c.), namely:
if $\Dcal_n\to\Dcal$ and $f_n\in S(\Dcal_n)$ with $f_n\to f$ in $\Hcal$, then $f\in S(\Dcal)$.
Algorithmic stability studies a single-valued output $A(\Dcal)\in\Ccal_{\Dcal}$
\citep{bousquet2002,ShalevShwartz2010}, implicit-bias analyses study a dynamical selection $f(t)\to f_\infty(\Dcal)$ within $\Ccal_{\Dcal}$ \citep{Soudry2018}, and robustness-based frameworks
evaluate risk or performance at a selected predictor $\hat f(\Dcal)$ \citep{ben-tal2009,rahimian2019};
all three work with single-valued selections (e.g., minimum-norm, regularized, or early-stopped solutions) and are most cleanly interpretable as statements about ERM when the underlying correspondence $S$ is PK-u.s.c.\ (or when one enforces a problem-intrinsic selection).
Indeed, if PK-u.s.c.\ fails, there exist $\Dcal_n\to\Dcal$ and $f_n\in\Ccal_{\Dcal_n}$
with $f_n\to f\notin\Ccal_{\Dcal}$ \citep{rockafellar1998}.
Consequently, a single-valued output may converge to a non-minimizer, and the resulting instability may reflect ill-posedness of the underlying correspondence rather than the selection mechanism alone.
In much of the literature, well-posedness is obtained after the fact by imposing a specific selection mechanism (e.g., via explicit regularization or algorithm-induced bias), or because the goal is to analyze the stability of a given algorithm.
Output-level stability controls a selection of $S(\Dcal)$ and does not, in general, characterize well-posedness of the correspondence; hence instability may be intrinsic or selection-induced, and conversely output stability does not preclude an ill-posed $S$.
In contrast, we study stability of the underlying set-valued ERM correspondence $S$ in full generality on a Hilbert space, prior to and independent of any particular selection rule.

Accordingly, we study ERM through a variational well-posedness lens in the non-unique convex regime.
Our contributions are as follows:
\begin{enumerate}[label=(\roman*),leftmargin=18pt,itemsep=1pt,topsep=2pt,parsep=0pt,partopsep=0pt]
\item \emph{Intrinsic stability.}
We identify intrinsic stability under data perturbations as Painlev\'e--Kuratowski upper semicontinuity of $S(\Dcal)=\argmin \Lcal_{\Dcal}$ 
(set-level Hadamard well-posedness).
\item \emph{Verifiable qualitative/quantitative regimes.}
Under Mosco perturbations and local boundedness of minimizers, $S$ is PK-u.s.c., minimal values are continuous,
and vanishing-gap near-minimizers are consistent; under quadratic growth (with uniform loss control) we obtain explicit deviation bounds.
\item \emph{Separation of intrinsic vs selection effects.}
This set-level viewpoint clarifies when output-level stability statements can be attributed to ERM rather than to a particular selection rule, disentangling solver-dependent selection effects from intrinsic variational ill-posedness.
\end{enumerate}

\section{Variational setting}

Let $(\Hcal,\|\cdot\|)$ be a real Hilbert space.
A dataset is an element of a metric space $(\mathbb{D},d)$.
To each $\Dcal\in\mathbb{D}$ we associate an empirical loss $\Lcal_{\Dcal}:\Hcal\to \R\cup\{+\infty\}$.

We first specify the class of empirical losses under consideration.

\begin{definition}[Admissible loss]
A functional $\Lcal:\Hcal \to \R\cup\{+\infty\}$ is \emph{admissible} if it is proper, convex, and lower semicontinuous (l.s.c.) with respect to the norm topology of $\Hcal$ \citep{rockafellar1998}.
\end{definition}

\begin{assumption}\label{ass:basic}
For every $\Dcal\in\mathbb{D}$, $\Lcal_{\Dcal}$ is admissible and the minimizer set
\mbox{$\Ccal_{\Dcal}:=\argmin_{f\in\Hcal}\Lcal_{\Dcal}(f)$} is nonempty.
\end{assumption}

For a given dataset $\mathcal{D}\in\mathbb{D}$, the empirical risk minimization (ERM) problem consists in solving
\begin{equation}\label{eq:erm_problem}
\min_{f\in\mathcal{H}} \mathcal{L}_{\mathcal{D}}(f).
\end{equation}
Since no uniqueness of minimizers is assumed or established under
Assumption~\ref{ass:basic}, the variational formulation naturally defines
a set-valued solution map rather than a single-valued one. We view learning as the set-valued solution map
\[
S:\mathbb D\rightrightarrows\mathcal H, \mathcal D\mapsto
S(\mathcal D):=\arg\min_{f\in\mathcal H}\mathcal L_{\mathcal D}(f).
\]
Convergence of datasets $\Dcal_n\to\Dcal$ is interpreted through the induced
\emph{vanishing variational perturbations} of the objectives, namely
\(
\Lcal_{\Dcal_n}=\Lcal_{\Dcal}+\delta_n,
\)
with $\delta_n\to 0$ in the chosen variational sense (Mosco in this work), i.e. $\Lcal_{\Dcal_n}\to \Lcal_{\Dcal}$.

Consequently, stability of learning must be
formulated at the level of solution sets, and is understood as the Hadamard well-posedness
of the solution correspondence $S(\mathcal D)$ under vanishing variational
perturbations of the empirical loss, independently of the mechanism generating them.
\begin{definition}[Stability with respect to data perturbations]
\label{def:data_stability}
The ERM problem \eqref{eq:erm_problem} is said to be \emph{stable at $\mathcal D$}
if for every sequence of datasets $\mathcal D_n\to\mathcal D$ and every sequence
$f_n\in S(\mathcal D_n)$ that converges strongly in $\Hcal$ to some $f$,
one has
\(
f\in S(\mathcal D).
\)
\end{definition}

\begin{definition}[Outer limit and upper semicontinuity]\label{def:pk_usc}
Let $\Dcal_n\to \Dcal$.
The \emph{Painlev\'e--Kuratowski outer limit} is
\[
\limsup_{n\to\infty}\Ccal_{\Dcal_n}
:=\Big\{f\in\Hcal:\ \exists n_k\to\infty,\ \exists f_{n_k}\in \Ccal_{\Dcal_{n_k}},\ f_{n_k}\to f\Big\}.
\]
We say that $S$ is \emph{Painlev\'e--Kuratowski upper semicontinuous} (PK-u.s.c.) at $\Dcal$ \citep{rockafellar1998,dontchev2009} if
\[
\limsup_{n\to\infty}\Ccal_{\Dcal_n}\subset \Ccal_{\Dcal}
\qquad\text{for every sequence }\Dcal_n\to\Dcal.
\]
\end{definition}
\noindent
Definition~\ref{def:data_stability} is equivalent to Painlev\'e--Kuratowski
upper semicontinuity of the solution correspondence $S$ at $\mathcal D$,
since any limit $f$ of a convergent sequence $f_n\in S(\mathcal D_n)$
belongs by definition to $\limsup_{n\to\infty} S(\mathcal D_n)$.
This equivalence provides the set-valued formulation of Hadamard stability
for ERM problems \citep{dontchev2009,bonnans2000}.
\begin{definition}[Hadamard well-posedness at the set level]
The ERM problem \eqref{eq:erm_problem} is said to be
\emph{well posed at $\mathcal{D}$} if:
(i) $\mathcal C_{\mathcal D}\neq\emptyset$, (ii) $\mathcal C_{\mathcal D}$ is closed, and (iii) $S$ is PK-u.s.c.\ at $\mathcal D$.

If, in addition, $\mathcal{C}_{\mathcal{D}}$ is bounded, the problem is said to be
\emph{strongly well posed at $\mathcal{D}$} \citep{bonnans2000,dontchev2009}.
\end{definition}
\noindent

We therefore take PK-u.s.c.\ as the intrinsic stability notion in the non-unique convex regime:
stronger set notions (e.g., pointwise or Hausdorff continuity) are non-canonical without additional structure or a selection rule
\citep{rockafellar1998,dontchev2009,bonnans2000}.

Stability analyses often apply to regularized selectors $f_\lambda(\mathcal D)$ (stable for fixed $\lambda>0$), which does not by itself certify set-level well-posedness of the underlying non-unique ERM as $\lambda\downarrow0$.
This can mislead interpretation: a stable selector does not imply a stable ERM correspondence, and an unstable output may simply expose an ill-conditioned selection mechanism rather than intrinsic ERM ill-posedness. Moreover, when a selection is analyzed without an explicit regularization parameter---that is, one studies a single-valued map $A(\Dcal)\in S(\Dcal)$ (algorithmic stability) or a dynamical limit $f_\infty(\Dcal)\in S(\Dcal)$ (implicit bias)---the conclusions remain statements about a selection.
Drawing conclusions about intrinsic ERM stability therefore requires PK-u.s.c.\ of the solution correspondence $S$; without an independently established set-level regularity, intrinsic ERM effects and selection-induced effects cannot be separated.

\begin{remark}\label{rem:flat_unbounded}
In overparameterized ERM, flat directions at the minimum may yield minimizer sets containing affine subspaces and hence unbounded solution sets
(e.g., least squares $\min_w\|Xw-y\|^2$ with $\ker(X)\neq\{0\}$: if $Xw^\star=y$, then $\argmin = w^\star+\ker(X)$).
Consequently, set-level Hadamard stability alone does not rule out minimizers drifting to infinity under small data perturbations.
\end{remark}

\section{Failure modes without boundedness or variational control}

The following example shows that even $C^1$ convex ERM objectives may admit minimizers
whose norm diverges under vanishing perturbations of the loss.

\begin{proposition}[Blow-up of minimizers under vanishing perturbations]\label{prop:blowup}
Let $\mathcal{H}=\mathbb{R}$ and define, for $\varepsilon>0$,
\begin{equation}\label{eq:Leps}
\mathcal{L}_{\varepsilon}(x):=\frac12\,(\varepsilon x-1)^2 .
\end{equation}
Then for every $\varepsilon>0$, $\mathcal{L}_{\varepsilon}$ is $C^{1}$ and convex,
\mbox{$\mathcal{C}_{\varepsilon}
=
\arg\min \mathcal{L}_{\varepsilon}
=
\left\{\frac{1}{\varepsilon}\right\}.$}
Moreover, as $\varepsilon\downarrow 0$, $\mathcal{L}_\varepsilon$ converges pointwise to
\(
\mathcal{L}_0(x)\equiv\frac12,
\)
while the minimizers satisfy $|x_\varepsilon|\to+\infty$.
Hence $(\mathcal C_\varepsilon)_{\varepsilon>0}$ is not uniformly bounded, so ERM may fail to be strongly well posed even in a smooth convex setting.
\end{proposition}

\begin{proof}
Differentiation gives $\mathcal{L}_\varepsilon'(x)=\varepsilon(\varepsilon x-1)$.
Thus the unique critical point solves $\varepsilon(\varepsilon x-1)=0$, hence $x_\varepsilon=1/\varepsilon$.
Since $\mathcal{L}_\varepsilon$ is a convex quadratic function, $x_\varepsilon$ is the unique minimizer.

For fixed $x\in\R$, $(\varepsilon x-1)^2\to1$ as $\varepsilon\downarrow0$, hence 
$\mathcal L_\varepsilon(x)=\tfrac12(\varepsilon x-1)^2\to \tfrac12=\mathcal L_0(x)$.
Moreover, $|x_\varepsilon|=|1/\varepsilon|\to+\infty$.
So for every $R>0$ there exists $\varepsilon>0$ such that $|x_\varepsilon|>R$; therefore the minimizers are not uniformly bounded.
\end{proof}

\begin{lemma}[Uniform boundedness excludes blow-up]
\label{lem:bounded_no_blowup}
Let $\mathcal D_n\to\mathcal D$ and assume that there exist $R>0$ and $n_0$ such that 
$\bigcup_{n\ge n_0}\mathcal C_{\mathcal D_n}\subset B(0,R)$.
Then for any sequence $f_n\in\mathcal C_{\mathcal D_n}$, the sequence $(f_n)$ is uniformly bounded.
In particular, escape to infinity is impossible.
\end{lemma}

\begin{proof}
For all $n\ge n_0$, $f_n\in B(0,R)$, hence $\|f_n\|\le R$.
\end{proof}

\begin{proposition}[Pointwise limits may destroy variational structure]
\label{prop:bounded_pointwise_degenerate}
Let $K:=[-R,R]\subset\mathbb R$ and consider the convex $C^1$ losses $\mathcal L_\varepsilon$ defined in \eqref{eq:Leps} for $\varepsilon>0$,
together with the hard constraint $\iota_K$, where $\iota_K$ denotes the indicator function of $K$.
Then:
\begin{enumerate}[label=(\roman*),leftmargin=18pt]
\item For every sufficiently small $\varepsilon>0$ (e.g., $\varepsilon<1/R$), the constrained problem has the unique minimizer
$\arg\min_{x\in K}\mathcal L_\varepsilon(x)=\{R\}$.
\item As $\varepsilon\downarrow 0$, one has pointwise convergence on $K$,
\(
\mathcal L_\varepsilon(x)\to \mathcal L_0(x)\equiv\frac12,
\)
and the limit problem is flat,
with 
\(
\arg\min_{x\in K}\mathcal L_0(x)=K.
\)
\item Consequently,
\[
\limsup_{\varepsilon\downarrow 0}\arg\min_{x\in K}\mathcal L_\varepsilon(x)
=\{R\}\subset K=\arg\min_{x\in K}\mathcal L_0(x),
\]
so Painlev\'e--Kuratowski u.s.c.\ holds at $\varepsilon=0$, but in a degenerate sense,
since the limit functional is flat on $K$ and the variational problem collapses.
\end{enumerate}
\end{proposition}

\begin{proof}
(i) On $K$ one has $\mathcal L_\varepsilon'(x)=\varepsilon(\varepsilon x-1)$ and
\mbox{$\mathcal L_\varepsilon''(x)=\varepsilon^2>0$}, so $\mathcal L_\varepsilon$ is strictly convex.
Its unconstrained minimizer is \mbox{$x_\varepsilon=1/\varepsilon>R$} for $\varepsilon<1/R$;
hence, for all sufficiently small $\varepsilon$, the constrained minimizer is attained at the boundary
and equals $R$. Uniqueness follows from strict convexity on $K$.

(ii) For each fixed $x\in K$, $(\varepsilon x-1)^2\to 1$, hence $\mathcal L_\varepsilon(x)\to\frac12$.
Since $\mathcal L_0$ is constant on $K$, every $x\in K$ is a minimizer.

(iii) The outer limit reduces to $\{R\}$ since the minimizer is constantly $R$ for small $\varepsilon$,
and the inclusion $\{R\}\subset K$ is immediate.
\end{proof}

Proposition~\ref{prop:blowup} shows that stability requires bounded minimizers to prevent blow-up, whereas
Proposition~\ref{prop:bounded_pointwise_degenerate} shows that boundedness alone does not prevent variational degeneration under pointwise perturbations. Hence set-level stability of minimizers must be paired with a loss convergence that preserves convex minimization, which motivates Mosco convergence in Hilbert spaces.

\section{Stability via Mosco continuity}

\begin{assumption}[Mosco continuity]\label{ass:mosco}
For every sequence \mbox{$\mathcal{D}_n \to \mathcal{D}$}, the convex l.s.c.\ functionals
$\mathcal{L}_{\mathcal{D}_n}$ Mosco-converge \citep{attouch1984} to $\mathcal{L}_{\mathcal{D}}$,
denoted by
\(
\mathcal{L}_{\mathcal{D}_n} \xrightarrow{Mosco} \mathcal{L}_{\mathcal{D}},
\)
meaning that:
\begin{samepage}
\begin{enumerate}[label=(M\arabic*),leftmargin=18pt]
\item\label{M1} (liminf under weak convergence)
for any $f_n\rightharpoonup f$ in $\Hcal$,
$\Lcal_{\Dcal}(f)\le \liminf_{n\to\infty}\Lcal_{\Dcal_n}(f_n)$;
\item\label{M2} (recovery sequence)
for every $f\in\Hcal$, there exists $f_n\to f$ in $\Hcal$ such that
\[
\limsup_{n\to\infty}\Lcal_{\Dcal_n}(f_n)\le \Lcal_{\Dcal}(f).
\]
\end{enumerate}
\end{samepage}
\end{assumption}

\noindent
Condition~\ref{M1} implies that variationally weakly minimizing limits are evaluated more favorably by the limit functional than by the approximating ones; in particular, perturbations cannot induce artificial minimizers outside the energy envelope of the limit problem.  Condition~\ref{M2} rules out
loss of minimizers by providing recovery sequences (in particular for minimizers of
$\mathcal L_{\mathcal D}$).

\begin{assumption}[Local boundedness of minimizers along convergent data]\label{ass:lbm}
For every sequence $\Dcal_n\to\Dcal$, there exist $R>0$ and $n_0$ such that
\(
\bigcup_{n\ge n_0}\Ccal_{\Dcal_n}\subset B(0,R).
\)
\end{assumption}

A standard sufficient condition for Assumption~\ref{ass:lbm} is equi-coercivity of $(\Lcal_{\Dcal_n})$ along $\Dcal_n\to\Dcal$:
indeed, equi-coercivity implies uniform boundedness of sublevel sets, hence in particular of the minimizers \citep{dalmaso1993}.

\begin{theorem}[Upper semicontinuity under Mosco continuity]\label{thm:usc_mosco}
Assume Assumptions~\ref{ass:basic}, \ref{ass:mosco}, and \ref{ass:lbm}.
Then $S$ is PK-u.s.c.\ and the minimal values $m(\Dcal):=\inf_{f\in\Hcal}\Lcal_{\Dcal}(f)$ satisfy $m(\Dcal_n)\to m(\Dcal)$ for every $\Dcal_n\to\Dcal$.
\end{theorem}

\begin{proof}
Let $\Dcal_n\to\Dcal$ and pick $f_n\in\Ccal_{\Dcal_n}$.
By Assumption~\ref{ass:lbm}, $(f_n)$ is bounded in $\Hcal$; since $\Hcal$ is Hilbert (hence reflexive), up to a subsequence (not relabeled) there exists $\bar f\in\Hcal$ such that $f_n\rightharpoonup \bar f$.
By the Mosco liminf condition~\ref{M1},
\[
\Lcal_{\Dcal}(\bar f)\le \liminf_{n\to\infty}\Lcal_{\Dcal_n}(f_n)=\liminf_{n\to\infty} m(\Dcal_n).
\]
Let $f^\star\in\Ccal_{\Dcal}$.
By the Mosco recovery condition~\ref{M2}, there exists $g_n\to f^\star$ in $\Hcal$ such that \[
\limsup_{n\to\infty}\Lcal_{\Dcal_n}(g_n)\le \Lcal_{\Dcal}(f^\star)=m(\Dcal).
\]
Since $m(\Dcal_n)\le \Lcal_{\Dcal_n}(g_n)$, it follows that 
\[
\limsup_{n\to\infty} m(\Dcal_n)\le m(\Dcal).
\]
Combining,
\[
m(\Dcal)\le \Lcal_{\Dcal}(\bar f)\le \liminf_{n\to\infty} m(\Dcal_n)\le \limsup_{n\to\infty} m(\Dcal_n)\le m(\Dcal),
\]
hence $\Lcal_{\Dcal}(\bar f)=m(\Dcal)$ and $\bar f\in\Ccal_{\Dcal}$; in particular \mbox{$m(\Dcal_n)\to m(\Dcal)$}.

Now let $f\in \limsup_{n\to\infty}\Ccal_{\Dcal_n}$.
Then there exist $n_k\to\infty$ and $f_{n_k}\in\Ccal_{\Dcal_{n_k}}$ such that $f_{n_k}\to f$ strongly; in particular $f_{n_k}\rightharpoonup f$.
Applying the previous argument to the subsequence $(\Dcal_{n_k},f_{n_k})$ yields $f\in\Ccal_{\Dcal}$.
Therefore $\limsup_{n\to\infty}\Ccal_{\Dcal_n}\subset\Ccal_{\Dcal}$, i.e.\ $S$ is PK-u.s.c.\ at $\Dcal$.
\end{proof}

Thus, under Mosco perturbations and local boundedness of minimizers, ERM is set-level Hadamard well posed.

\begin{corollary}[Stability of $\varepsilon_n$-minimizers]\label{cor:eps_minimizers}
Assume the hypotheses of Theorem~\ref{thm:usc_mosco}.
Let $\Dcal_n\to\Dcal$ and let $\varepsilon_n\downarrow 0$.
Let $(f_n)\subset\Hcal$ satisfy
\(
\Lcal_{\Dcal_n}(f_n)\le m(\Dcal_n)+\varepsilon_n.
\)
If $f_n\to f$ strongly in $\Hcal$, then $f\in\Ccal_{\Dcal}$ (equivalently $\Lcal_{\Dcal}(f)=m(\Dcal)$).
\end{corollary}

\begin{proof}
By Mosco liminf~\ref{M1},
\[
\Lcal_{\Dcal}(f)\le \liminf_{n\to\infty}\Lcal_{\Dcal_n}(f_n)
\le \liminf_{n\to\infty}\bigl(m(\Dcal_n)+\varepsilon_n\bigr)
= \liminf_{n\to\infty} m(\Dcal_n).
\]
By Theorem~\ref{thm:usc_mosco}, one has $m(\Dcal_n)\to m(\Dcal)$; hence
\[
\liminf_{n\to\infty} m(\Dcal_n)=m(\Dcal),
\]
and therefore
\[
\Lcal_{\Dcal}(f)\le m(\Dcal).
\]
Since always $\Lcal_{\Dcal}(f)\ge m(\Dcal)$, we obtain $\Lcal_{\Dcal}(f)=m(\Dcal)$,
i.e.\ $f\in\Ccal_{\Dcal}$.
\end{proof}

The stability extends to numerically computed solutions (near-minimizers): it is natural to model algorithmic outputs as $f_n$ with a vanishing optimality gap
$\Lcal_{\Dcal_n}(f_n)-m(\Dcal_n)\le \varepsilon_n\to0$ (early stopping, numerical tolerances, SGD).
If moreover $f_n\to f$ strongly (e.g., under coercivity, regularization, or a compact constraint), then $f\in\arg\min \Lcal_{\Dcal}$.

In practice, Mosco perturbations are ensured for instance, when the data enter $\Lcal_{\Dcal}$ only through convergent convex l.s.c.\ summaries (e.g., empirical measures, moments, Gram matrices), with uniformly bounded (near-)minimizers (e.g., equi-coercivity, compact constraints, or uniform coercivity/strong convexity).

\begin{remark}[Hard constraints and minimization at infinity]
A hard constraint $\iota_{B_R}$ can stabilize ERM without changing the solution set only when the problem admits bounded minimizing sequences: if $\exists R$ such that $m_R:=\inf_{\|f\|\le R}\Lcal_{\Dcal}(f)=\inf\Lcal_{\Dcal}$, then for any $R'>R$ one has
\(
\arg\min \Lcal_{\Dcal}\;=\;\arg\min\big(\Lcal_{\Dcal}+\iota_{B_{R'}}\big),
\)
and boundedness is enforced. Conversely, if $m_R>\inf\Lcal_{\Dcal}$ for all $R$ (minimization at infinity), no finite ball constraint can enforce boundedness without altering the objective.
\end{remark}

In the Mosco--boundedness regime, the ERM correspondence $S$ is PK-u.s.c.\ at $\Dcal$.
This isolates intrinsic ERM stability and provides a sound basis for subsequent algorithmic and statistical analyses.

\section{Quantitative stability from quadratic growth}

Mosco continuity yields qualitative stability; within this regime, we derive a quantitative bound under quadratic growth and uniform control of loss perturbations on bounded sets.

\begin{definition}[Quadratic growth]\label{def:qg}
Let $\Lcal:\Hcal\to\R\cup\{+\infty\}$ be admissible, with minimizers $\Ccal=\argmin\Lcal$, minimal value $m:=\inf \Lcal$,
and $\dist(f,\Ccal):=\inf_{g\in\Ccal}\|f-g\|$.
We say that $\Lcal$ satisfies \emph{quadratic growth} \citep{rockafellar1998}
 on a set $B\subset\Hcal$ with constant $\mu>0$ if
\(
\Lcal(f)-m\ \ge\ \frac{\mu}{2}\,\dist(f,\Ccal)^2
\qquad \forall f\in B.
\)
\end{definition}

\begin{remark}[Quadratic growth as an error bound]
For convex losses, quadratic growth is the weakest second-order error-bound condition that converts objective gaps into $\dist(\cdot,\Ccal_{\Dcal})$ control; linear growth can fail along flat directions, while higher-order growth requires additional sharpness.
\end{remark}

\begin{assumption}[Uniform convergence on bounded sets]\label{ass:unif_ball}
For every sequence $\Dcal_n\to\Dcal$ and every $R>0$, letting
\mbox{$\varepsilon_n(R):=\sup_{\|f\|\le R}\big|\Lcal_{\Dcal_n}(f)-\Lcal_{\Dcal}(f)\big|$},
one has $\varepsilon_n(R)\to0$ as $n\to\infty$.
\end{assumption}

\begin{theorem}[Quantitative stability under quadratic growth]\label{thm:qg_quant}
Fix $\Dcal\in\mathbb{D}$ and assume $\Lcal_{\Dcal}$ satisfies quadratic growth on $B(0,R)$ with constant $\mu>0$.
Let $\Dcal_n\to\Dcal$ be a sequence such that
\mbox{$\Ccal_{\Dcal}\ \cup\ \bigcup_{n\ge n_0}\Ccal_{\Dcal_n}\ \subset\ B(0,R)$}.
Assume Assumption~\ref{ass:unif_ball} holds on $B(0,R)$.
Then for any choice $f_n\in \Ccal_{\Dcal_n}$ (for $n$ large),
\[
\dist(f_n,\Ccal_{\Dcal})
\ \le\
\sqrt{\frac{4}{\mu}\,\varepsilon_n(R)}.
\]
\end{theorem}

\begin{proof}
Let $f_n\in \Ccal_{\Dcal_n}$ and let $f^\star\in\Ccal_{\Dcal}$.
By Assumption~\ref{ass:unif_ball} (applied on $B(0,R)$),
\[
\Lcal_{\Dcal}(f_n)
\le \Lcal_{\Dcal_n}(f_n)+\varepsilon_n(R)
= m(\Dcal_n)+\varepsilon_n(R).
\]
Also,
\[
m(\Dcal_n)
\le \Lcal_{\Dcal_n}(f^\star)
\le \Lcal_{\Dcal}(f^\star)+\varepsilon_n(R)
= m(\Dcal)+\varepsilon_n(R).
\]
Therefore $\Lcal_{\Dcal}(f_n)-m(\Dcal)\le 2\varepsilon_n(R)$.
Quadratic growth on $B(0,R)$ yields
\mbox{$\frac{\mu}{2}\dist(f_n,\Ccal_{\Dcal})^2 \le 2\varepsilon_n(R)$},
hence the stated bound.
\end{proof}

This estimate yields $\dist(f_n,\mathcal C_{\mathcal D})=O(\sqrt{\varepsilon_n(R)})$ as $n\to\infty$ (for fixed $R$ and $\mu>0$), providing an explicit quantitative link between loss perturbations and ERM solution stability under quadratic growth.

\section{Strongly convex regularization as a stabilization mechanism}

Strongly convex regularization canonically enforces quadratic growth; e.g., $\Rcal(f)=\tfrac12\|f\|^2$ gives Tikhonov (ridge) regularization \citep{rockafellar1998}. This strong convexity yields quantitative stability for the regularized ERM.

\begin{definition}[Uniformly strongly convex regularizers]\label{def:reg_family}
Let $\{\Rcal_\eta\}_{\eta\in\Eta}$ be proper, l.s.c., convex functionals on $\Hcal$.
We say that the family is \emph{uniformly $\alpha$-strongly convex} if for every $\eta\in\Eta$,
$\Rcal_\eta$ is $\alpha$-strongly convex with the same $\alpha>0$.
\end{definition}

\begin{proposition}[Strong convexity implies quadratic growth and uniqueness]\label{prop:strong_convex_stab}
Let $\Lcal_{\Dcal}$ be admissible and let $\Rcal_\eta$ be uniformly $\alpha$-strongly convex.
Assume that there exists $f_0\in\Hcal$ such that \mbox{$\Lcal_{\Dcal}(f_0)<+\infty$} and $\Rcal_\eta(f_0)<+\infty$.
Fix $\lambda>0$ and define
\[
f^\star_{\Dcal,\eta}
:=
\argmin_{f\in\Hcal}\Big(\Lcal_{\Dcal}(f)+\lambda\,\Rcal_\eta(f)\Big).
\]
Then the minimizer is unique and the objective satisfies quadratic growth with constant $\mu=\lambda\alpha$ \citep{rockafellar1998}
:
\[
\Big(\Lcal_{\Dcal}+\lambda\Rcal_\eta\Big)(f)-\Big(\Lcal_{\Dcal}+\lambda\Rcal_\eta\Big)(f^\star_{\Dcal,\eta})
\ \ge\ \frac{\lambda\alpha}{2}\,\|f-f^\star_{\Dcal,\eta}\|^2
\qquad \forall f\in\Hcal.
\]
\end{proposition}

\begin{proof}
Since $\mathcal R_\eta$ is uniformly $\alpha$-strongly convex, $\lambda\mathcal R_\eta$ is
$\lambda\alpha$-strongly convex. Hence \mbox{$F:=\mathcal L_{\mathcal D}+\lambda\mathcal R_\eta$} is proper, l.s.c., and $\lambda\alpha$-strongly convex on $\Hcal$.
Therefore $F$ admits a unique minimizer $f^\star_{\mathcal D,\eta}\in\Hcal$
\citep{rockafellar1998}.

By $\lambda\alpha$-strong convexity, for all $f\in\Hcal$,
\[
F(f)\ge F(f^\star_{\mathcal D,\eta})+\frac{\lambda\alpha}{2}\,\|f-f^\star_{\mathcal D,\eta}\|^2,
\]
which is the claimed quadratic growth inequality (with \mbox{$\mu=\lambda\alpha$}).
\end{proof}

\section{Conclusion}
\label{sec:conclusion}

In this work, we establish PK-u.s.c.\ set-level well-posedness as the intrinsic stability notion for non-unique convex ERM.
We provide verifiable Mosco-based conditions (with bounded \mbox{(near-)minimizers)} and a quadratic-growth quantitative estimate, which offers a principled attribution of output (in)stability: it distinguishes instability of the ERM correspondence from instability induced by a solver, or by a regularized selector $f_\lambda$ (stable for fixed $\lambda>0$ but potentially ill-conditioned as $\lambda\downarrow0$).
Extending these guarantees to nonconvex ERM is a natural next step.

\begingroup
\small
\setlength{\itemsep}{0pt}
\setlength{\parskip}{0pt}
\setlength{\parsep}{0pt}
\setlength{\bibsep}{2pt}

\endgroup


\begin{thebibliography}{99}
\expandafter\ifx\csname url\endcsname\relax
  \def\url#1{\texttt{#1}}\fi
\expandafter\ifx\csname urlprefix\endcsname\relax\def\urlprefix{URL }\fi
\expandafter\ifx\csname href\endcsname\relax
  \def\href#1#2{#2} \def\path#1{#1}\fi

\bibitem{Vapnik1998}
V.~N. Vapnik, Statistical Learning Theory, Wiley, New York, 1998.

\bibitem{ShalevShwartz2014}
S.~Shalev-Shwartz, S.~Ben-David, Understanding Machine Learning: From Theory to Algorithms, Cambridge University Press, Cambridge, 2014.

\bibitem{rockafellar1998}
R.~T. Rockafellar, R.~J.-B. Wets, Variational Analysis, Vol. 317 of Grundlehren der mathematischen Wissenschaften, Springer, Berlin, 1998.

\bibitem{Pineau2021}
J.~Pineau, P.~Vincent-Lamarre, K.~Sinha, V.~Lariviere, A.~Beygelzimer, F.~d'Alche Buc, E.~Fox, H.~Larochelle, Improving reproducibility in machine learning research, Journal of Machine Learning Research 22~(164) (2021) 1--20.

\bibitem{Zhang2017}
C.~Zhang, S.~Bengio, M.~Hardt, B.~Recht, O.~Vinyals, Understanding deep learning requires rethinking generalization, in: Proceedings of the International Conference on Learning Representations (ICLR), 2017.

\bibitem{Belkin2019}
M.~Belkin, D.~Hsu, S.~Ma, S.~Mandal, Reconciling modern machine-learning practice and the classical bias--variance trade-off, Proceedings of the National Academy of Sciences 116~(32) (2019) 15849--15854.
\newblock \href {https://doi.org/10.1073/pnas.1903070116} {\path{doi:10.1073/pnas.1903070116}}.

\bibitem{Neyshabur2015}
B.~Neyshabur, R.~Tomioka, N.~Srebro, Norm-based capacity control in neural networks, in: Proceedings of the 28th Conference on Learning Theory ({COLT}), Vol.~40 of Proceedings of Machine Learning Research, 2015, pp. 1376--1401.

\bibitem{attouch1984}
H.~Attouch, Variational Convergence for Functions and Operators, Pitman, Boston, 1984.

\bibitem{dalmaso1993}
G.~Dal~Maso, An Introduction to {\ensuremath{\Gamma}}-Convergence, Vol.~8 of Progress in Nonlinear Differential Equations and Their Applications, Birkh{\"a}user, Boston, 1993.

\bibitem{bousquet2002}
O.~Bousquet, A.~Elisseeff, Stability and generalization, Journal of Machine Learning Research 2 (2002) 499--526.

\bibitem{ShalevShwartz2010}
S.~Shalev-Shwartz, O.~Shamir, N.~Srebro, K.~Sridharan, Learnability, stability and uniform convergence, Journal of Machine Learning Research 11 (2010) 2635--2670.

\bibitem{Soudry2018}
D.~Soudry, E.~Hoffer, M.~Nacson, S.~Gunasekar, N.~Srebro, The implicit bias of gradient descent on separable data, Journal of Machine Learning Research 19~(70) (2018) 1--57.

\bibitem{ben-tal2009}
A.~Ben-Tal, L.~El~Ghaoui, A.~Nemirovski, Robust Optimization, Princeton University Press, Princeton, 2009.

\bibitem{rahimian2019}
H.~Rahimian, S.~Mehrotra, Distributionally robust optimization: A review (2019).
\newblock \href {http://arxiv.org/abs/1908.05659} {\path{arXiv:1908.05659}}.

\bibitem{dontchev2009}
A.~L. Dontchev, R.~T. Rockafellar, Implicit Functions and Solution Mappings: A View from Variational Analysis, Springer, New York, 2009.

\bibitem{bonnans2000}
J.~F. Bonnans, A.~Shapiro, Perturbation Analysis of Optimization Problems, Springer, New York, 2000.

\end{thebibliography}
\end{document}